\documentclass[letterpaper, 10 pt, conference]{ieeeconf}  

\IEEEoverridecommandlockouts                              
\usepackage[utf8]{inputenc}
\usepackage[T1]{fontenc}
\usepackage{amsmath,bm}
\usepackage{amsfonts}

\usepackage{amsthm}

\let\labelindent\relax  
\usepackage{enumitem}

\usepackage[english]{babel}
\usepackage[colorlinks=true, citecolor=black, linkcolor=black, urlcolor=black]{hyperref}
\usepackage{cite}
\usepackage{threeparttable} 

\usepackage{multirow}

\usepackage{graphicx}
\usepackage[font=footnotesize]{caption}
\usepackage{subcaption}
\usepackage{titlesec}
\usepackage{nth}
\usepackage{cleveref} 
\usepackage{lipsum}  
\usepackage[dvipsnames]{xcolor}
\usepackage{colortbl}
\definecolor{light-gray}{gray}{0.93}
\definecolor{semilight-gray}{gray}{0.8}

\usepackage{amssymb}  

\newtheoremstyle{defiStyle}  
  {3pt} {3pt} {\normalfont} {} {\bfseries\itshape} {:} { }  
  {\thmname{#1} \thmnumber{#2}\textit{\thmnote{ (#3)}}}  

\theoremstyle{defiStyle}  
\newtheorem{defi}{Definition}  

\theoremstyle{thmStyle}
\newtheorem{thm}{Theorem}  
\newtheorem{lemma}{Lemma}  
\newtheorem{remark}{Remark}  
\newtheorem{prob}{Problem}

\newtheorem{corollary}{Corollary}

\makeatletter
\renewenvironment{proof}[1][\proofname]{%
  \par\pushQED{\qed}%
  \normalfont\topsep6pt \trivlist\item[\hskip\labelsep\itshape#1.]\ignorespaces
}{%
  \hfill\QED\endtrivlist\@endpefalse
}
\makeatother

\renewcommand{\proofname}{\textit{Proof}}

\newcommand{\mdspace}{$\mathcal{M}_\mathcal{D}$-space}
\newcommand{\cobs}{CO}

\title{\LARGE \bf
Control Barrier Functions via Minkowski Operations \\ for Safe Navigation among Polytopic Sets
}

\author{Yi-Hsuan Chen$^{1}$, Shuo Liu$^{2}$, Wei Xiao$^{3}$, Calin Belta$^{4}$, Michael Otte$^{1}$
\thanks{$^{1}$Y. Chen and M. Otte are with the Dept. of Aerospace Engineering, University of Maryland, College Park, MD, USA. {\tt \{yhchen91, otte\}@umd.edu}}%
\thanks{$^{2}$S. Liu is with the Dept. Mechanical Engineering, Boston University, Brookline, MA, 02215, USA. {\tt liushuo@bu.edu}}%
\thanks{$^{3}$W. Xiao is with the Computer Science and Artificial Intelligence Lab, Massachusetts Institute of Technology, Cambridge, MA, USA. {\tt weixy@mit.edu}}
\thanks{$^{4}$C. Belta is with the Dept. Electrical and Computer Engineering and Dept. Computer Science, University of Maryland, College Park, MD, USA. {\tt calin@umd.edu}}
}

\newcommand{\removeForSpace}[1]{}

\begin{document}


\maketitle
\begin{abstract}
Safely navigating around obstacles while respecting the dynamics, control, and geometry of the underlying system is a key challenge in robotics. Control Barrier Functions (CBFs) generate safe control policies by considering system dynamics and geometry when calculating safe forward-invariant sets. Existing CBF-based methods often rely on conservative shape approximations, like spheres or ellipsoids, which have explicit and differentiable distance functions. In this paper, we propose an optimization-defined CBF that directly considers the exact Signed Distance Function (SDF) between a {\it polytopic} robot and {\it polytopic} obstacles.
Inspired by the Gilbert-Johnson-Keerthi (GJK) algorithm, we formulate both (i) minimum distance and \mbox{(ii) penetration} depth between polytopic sets as convex optimization problems in the space of Minkowski difference operations (the \mdspace). 
Convenient geometric properties of the \mdspace ~enable the derivatives of implicit SDF between two polytopes to be computed via differentiable optimization. 
%
%
%
%
We demonstrate the proposed framework in three scenarios including pure translation, initialization inside an unsafe set, and multi-obstacle avoidance. These three scenarios highlight the generation of a non-conservative maneuver, a recovery after starting in collision, and the consideration of multiple obstacles via pairwise CBF constraints, respectively. A video of all simulation animations can be found at
\textnormal{\url{https://youtu.be/3Dh0gtDW8bE}}
.
\end{abstract}
\section{Introduction}
\begin{figure}[ht!]
    \centering
    \includegraphics[width=0.9\linewidth]{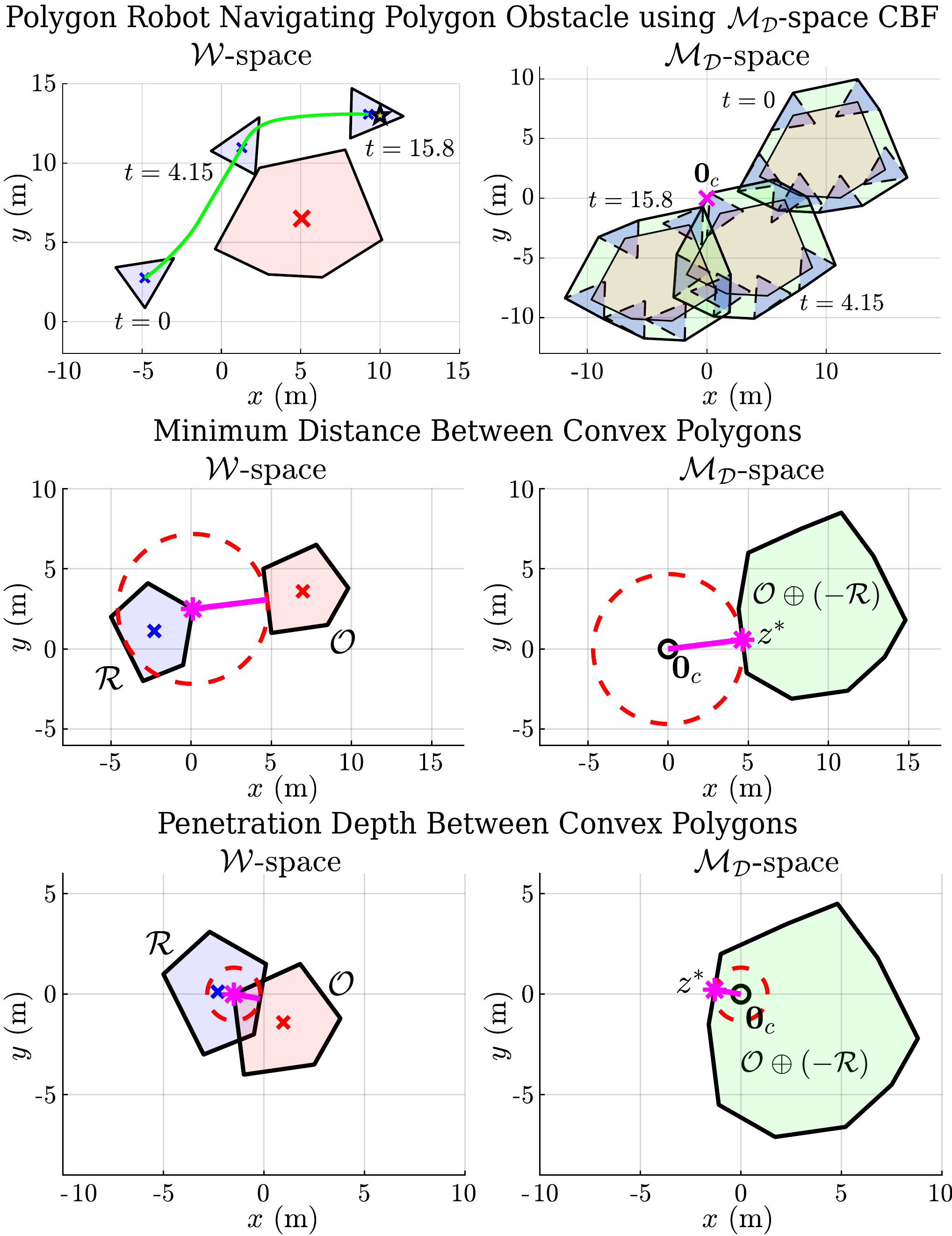}
    \caption{ {\bf Top:} A polytopic robot navigating around a polytopic obstacle in $\mathcal{W}$-space (Left) and the corresponding 
    Minkowski differences 
    in \mdspace~(Right). The proposed CBF method leverages the fact that signed distance between polytopes in $\mathcal{W}$-space can be derived in \mdspace. 
    {\bf Middle/Bottom:} Distances between disjoint (Middle) and intersecting (Bottom) convex sets in $\mathcal{W}$-space (left), and the same distances in \mdspace~(Right).
    In \mdspace~the origin $\bm{0}_c$ lies outside or inside of the Minkowski difference if the signed distance is positive or negative, respectively. $z^*$ is the critical point, whose derivative is used to enforce CBF constraints. 
    }
    \label{fig: sd_CSO}
\end{figure}
Safely navigating an environment with obstacles is a fundamental challenge in robotics. Navigational solutions must respect the dynamics and control constraints of the vehicle and the geometric constraints of the vehicle and obstacles. \textcolor{black}{Control Barrier Functions (CBFs) guarantee safety for control-affine systems by providing forward-invariant safe sets and combine naturally with CLFs, yielding a sequence of pointwise quadratic programs (QP) \cite{Ames.etal.ECC19,Ames.etal.TAC17}.} QPs are computationally efficient and well-suited to real-time applications. Further advancements in safety-critical control with CBFs have introduced formulations for high relative degree constraints \cite{Nguyen.Sreenath.ACC16, Xiao.Belta.TAC22, Liu.etal.23, 10886196,liu2025auxiliary}. 

CBFs are used for obstacle avoidance across a range of applications \cite{Agrawal.Sreenath.RSS17,Thyri.etal.CCTA20,Khazoom.etal.Humanoids22,Thirugnanam.etal.ICRA22,Tayal.etal.ACC24,Liu.Mao.arXiv24,Molnar.etal.TCST25}, yet most CBFs approximate robots and/or obstacles by smooth shapes, i.e., hyperspheres \cite{Chen.etal.TCST18} and ellipsoids \cite{Verginis.Dimarogonas.LCSS19,Funada.etal.TCST25}. Hyperspheres and ellipsoids are popular because they admit closed-form, differentiable robot-obstacle distance, but they are often overly conservative approximations of real-world objects; such objects are, for example, more accurately approximated by polytopes. 

In this paper we present a CBF-based framework that considers {\it polytopic} robots and {\it polytopic} obstacles directly. Our method computes the exact signed distance function (SDF) between two polytopes by solving a convex optimization problem in the dual space that is defined by the Minkowski difference between the two polytopes (see Fig.~\ref{fig: sd_CSO}). The resulting robot-to-obstacle SDF is differentiable almost everywhere (a.e.).

Our work builds on a foundation of prior work \cite{Thirugnanam.etal.ACC22, Thirugnanam.etal.arXiv23, Singletary.etal.RAL22,Peng.etal.ICRA23,Wei.etal.RAL24,Wu.etal.optfree.arxiv25} that considers continuous-time CBFs for collision avoidance between polytopes/polygons. The first study of CBFs \cite{Thirugnanam.etal.ACC22} uses the Minimum Distance Function (MDF), solved by duality-based convex optimization, to construct a nonsmooth CBF (NCBF). In \cite{Thirugnanam.etal.ACC22, Thirugnanam.etal.arXiv23} the exact derivative is applied in the case of strongly convex sets (but not polytopes, which are weakly convex), while a conservative lower bound on the CBF derivative is used for polytopes. A CBF-based filtering strategy for manipulators  \cite{Singletary.etal.RAL22}
uses the SDF and a local linear approximation of the CBF gradient (approximation error considered to ensure safety). An optimization-free smooth approx-SDF CBF is proposed in \cite{Wu.etal.optfree.arxiv25}, where conservatism is reduced via hyperparameter tuning. In \cite{Wei.etal.RAL24} CBFs are constructed using differentiable optimization of the minimum scaling factor (another metric for collision detection), and polygons are approximated by strictly convex shapes. Polygons are similarly approximated by polynomial shapes using logistic regression in \cite{Peng.etal.ICRA23}. Our work differs from prior work \cite{Thirugnanam.etal.ACC22, Thirugnanam.etal.arXiv23, Singletary.etal.RAL22,Peng.etal.ICRA23,Wei.etal.RAL24,Wu.etal.optfree.arxiv25}  in that we solve the SDF via convex optimization and compute its derivative using differentiable optimization \textit{without approximation} by leveraging the configuration obstacle defined via the Minkowski difference. Table \ref{table1} provides a conceptual comparison with existing works.

The Minkowski difference is a fundamental tool in computational geometry for collision detection. It was first used in the Gilbert-Johnson-Keerthi (GJK) algorithm \cite{gjk} to compute the distance
between convex sets, and later extended by the Expanding Polytope Algorithm (EPA) \cite{epa} to compute penetration depth. Due to its effectiveness in encoding robot's geometry and configuration with obstacles, it has been widely used in sampling-based \cite{Lien.08},~\cite{Ruan.etal.TRO23} and optimization-based motion planning \cite{Lutz.Meurer.CCTA21,Montaut.etal.RSS22,Guthrie.CDC22}. In this article, we refer to the space in which Minkowski difference lies as the ``\textit{Minkowski-difference-space}''\footnote{While some authors refer to \mdspace ~as ``configuration space,'' we use the term  ``Minkowski-difference-space'' to distinguish it from the separate ``configuration space'' ($\mathcal{C}$-space) typically considered in motion planning. It is also chosen to be similar to the ``collision space'' defined in \cite{Montaut.etal.RSS22}. The more succinct ``Minkowski space'' has a specific different meaning in physics.}, denoted as \mdspace. 

While CBFs have previously considered \mdspace~for collision avoidance \cite{Wu.etal.optfree.arxiv25}, prior work relies on an approximation of the SDF by its lower bound. 
The integration of exact SDF computation via convex optimization with CBFs remains an open problem. To the best of our knowledge, we are the first to bridge this gap by combining the Minkowski difference, convex optimization, and CBFs in a unified framework that considers exact safety constraints.

The main contributions of this paper are as follows:
\begin{itemize}
    \item We propose a novel approach to synthesize 
    CBFs for obstacle avoidance between polytopic sets using the exact signed distance. This is achieved by solving the distance problem in the \mdspace. 
    The derivatives are computed via differentiable optimization and are used to evaluate the derivatives of CBFs.
    \item We present an exact formulation for computing penetration depth as convex optimization. 
    This reformulation allows us to use SDF as a CBF, as both minimum distance and penetration depth share the same differentiable optimization framework.
    \item We demonstrate 
    the proposed method 
    on single-integrator and unicycle models navigating around obstacles, including scenarios with initial collisions to validate the penetration depth formulation.
\end{itemize}

\section{Preliminaries} \label{sec:preliminaries}
In this section we present background on CBFs (\ref{subsec:prelim_cbf}), Minkowski operations (\ref{subsec:prelim_minkowski}), and signed distance  (\ref{subsec:prelim_sd}).
\subsection{Control Barrier Functions}\label{subsec:prelim_cbf}
Consider a nonlinear affine control system:
\begin{align}
    \dot{\bm{x}}=f(\bm{x})+g(\bm{x})\bm{u},\label{eq: ctrl-affine}
\end{align}
where $\bm{x}\in \mathcal{X}\subset\mathbb{R}^n$ ($\mathcal{X}$ denotes a closed state constraint set), $f:\mathbb{R}^n\rightarrow \mathbb{R}^n$ and $g:\mathbb{R}^n\rightarrow \mathbb{R}^{n\times q}$ are locally Lipschitz, and $\bm{u}\in \mathcal{U}\subset\mathbb{R}^q$ ($\mathcal{U}$ denotes a closed control constraint set).
\begin{defi}[\it Forward invariant set]
    \normalfont A set $C\subset\mathbb{R}^n$ is forward invariant for system \eqref{eq: ctrl-affine} if its solutions starting at $\bm{x}(t_0)\in C$ satisfy $\bm{x}(t)\in C,\,\,\forall t\geq t_0$.
\end{defi}
Let $C$ be the safe set, defined as the 0-superlevel set of \textcolor{black}{continuously differentiable}
$h:\mathcal{X}\subset\mathbb{R}^n\rightarrow\mathbb{R}$
:
\begin{align}\begin{split} 
    C &= \{\bm{x}\in \mathcal{X}\subset\mathbb{R}^n:h(\bm{x})\geq 0 \}.
    \label{eq: safeset_def}
\end{split}\end{align}
\begin{defi}[\it CBF \cite{Ames.etal.TAC17}]
    \normalfont
    Given a set $C$ as in \eqref{eq: safeset_def}, $h(\bm{x})$ is a control barrier function (CBF) for system \eqref{eq: ctrl-affine} if there exists an 
    extended class $\mathcal{K}_\infty$ function\footnote{%
    An extended class $\mathcal{K}_\infty$ function is a function $\alpha:\mathbb{R}\rightarrow\mathbb{R}$ that is strictly increasing and $\alpha(0)=0$, see, e.g., \cite{Ames.etal.ECC19}.
    } 
    $\alpha$ such that, for all $\bm{x}\in \mathcal{X}$,
    \begin{align}
    \textstyle\sup_{\bm{u}\in\mathcal{U}}      [L_fh(\bm{x})+L_gh(\bm{x})\bm{u}+\alpha(h(\bm{x}))]\geq0,\label{eq: cbf_def}
    \end{align}
    where $L_f(\cdot), L_g(\cdot)$ denote the Lie derivatives along $f$ and $g$, respectively. It is assumed that $L_gh(\bm{x})\neq 0$ when $h(\bm{x})= 0$.\label{def: cbf}
\end{defi}
%
\begin{thm}[Safety Guarantee from \cite{Ames.etal.TAC17}]
    Given a CBF $h$ with the associated set $C$ from \eqref{eq: safeset_def}, any Lipschitz continuous controller $\bm{u}(t)\in \mathcal{U}, t\geq t_0$ that satisfies \eqref{eq: cbf_def} renders the set $C$ forward invariant for \eqref{eq: ctrl-affine}.
    \label{thm: safety-guarantee}
\end{thm}

\begin{remark}
    \textcolor{black}{Our SDF is continuously differentiable w.r.t. spatial variables but (only) differentiable a.e. w.r.t. orientation (due to active-set changes) and (only) when the robot and obstacle are not in collision. Similarly, in the cases of contact/penetration, the SDF is locally Lipschitz, and thus (only) differentiable a.e., such that non-differentiability occurs on a measure-zero set (vertices, interior medial axis, rotation-induced active-set switches). The classical CBF in Def.~\ref{def: cbf} does not apply at those (measure-zero set of) points. Consequently, the proofs in this paper (which build from the classical approach) fully apply to the non-collision case for single-integrator dynamics. We conjecture that a parallel analysis can be obtained a.e. for the collision and orientation-dependent cases; however, a rigorous nonsmooth analysis such as in \cite{Glotfelter.etal.LCSS17, Thirugnanam.etal.arXiv23} is planned for future work.}
\end{remark}

\begin{defi} [\it Control Lyapunov function \cite{Ames.CLF.CDC12}]  
\normalfont
A continuously differentiable function $V: \mathcal{X}\to \mathbb{R}$ is a globally and exponentially stabilizing control Lyapunov function (CLF) for system \eqref{eq: ctrl-affine} if there exist $c_1 > 0, c_2 > 0, c_3 > 0$ such that $\forall \bm{x} \in \mathcal{X}$, $c_1 \|\bm{x}\|^2 \leq V(\bm{x}) \leq c_2 \|\bm{x}\|^2$ and
    \begin{align}
    \textstyle
        \inf_{\bm{u}\in\mathcal{U}}      [L_fV(\bm{x})+L_gV(\bm{x})\bm{u}+c_3V(\bm{x})]\leq0.
        \label{eq: clf_def}
    \end{align}
\end{defi}

\begin{table*}[t!]
    \centering
    \caption{Conceptual comparison of existing CBF-based collision avoidance methods for polytopic sets.
    }
    \vspace{-0.5em}
    \resizebox{\textwidth}{!}{
    \begin{tabular}{|c|cccccc|}
    \hline
        \textbf{Approaches}
        & \textbf{Metric}
        & \textbf{Explicitness} 
        & \textbf{Smoothness}
        & \textbf{Conservatism} 
        & \textbf{Obstacle Defined in} 
        & \begin{tabular}{@{}c@{}}\textbf{Collision Recovery} \\[-0.2em] \textbf{Demonstrated}\end{tabular}
        \\
    \hline
    \hline  
        MDF-based CBF \cite{Thirugnanam.etal.ACC22, Thirugnanam.etal.arXiv23} 
        & Distance
        & Implicit
        & Nonsmooth
        & \begin{tabular}{@{}c@{}}Nonconservative (strongly convex sets) \\[-0.2em] Conservative (polytopes)\end{tabular}
        & Workspace
        & No\\
    \hline
        SDF-linearized CBF \cite{Singletary.etal.RAL22} 
        & Distance 
        & Implicit
        & Smooth
        & Conservative
        & Workspace 
        & No\\
    \hline
        Polynomial CBF \cite{Peng.etal.ICRA23} 
        & Distance 
        & Explicit 
        & Smooth
        & Tunable 
        & Workspace 
        & No\\
    \hline
        Diff-opti CBF \cite{Wei.etal.RAL24}
        & Scaling factor
        & Implicit 
        & Smooth
        & Tunable 
        & Workspace
        & No\\
    \hline
        Opti-free CBF \cite{Wu.etal.optfree.arxiv25} 
        & Distance
        & Explicit
        & Smooth
        & Tunable
        & \mdspace
        & No \\
    \hline
        \rowcolor{light-gray} {Ours} 
        & Distance 
        & Implicit 
        & Nonsmooth
        & Nonconservative
        & \mdspace 
        & Yes \\
    \hline
       \multicolumn{7}{p{7.5in}}{} \\[-0.75em]
       \multicolumn{7}{p{7.5in}}{\textit{Implicit} indicates that CBFs are defined via optimization \cite{Thirugnanam.etal.ACC22,Peng.etal.ICRA23,Wei.etal.RAL24}, or computed algorithmically \cite{Singletary.etal.RAL22}. \textit{Conservatism} indicates whether the method uses approximations or lower bounds in its formulation. \textit{Collision recovery} refers to the ability to recover from an initial collision with an obstacle; this property is theoretically enabled by SDF-based CBFs \cite{Singletary.etal.RAL22,Wu.etal.optfree.arxiv25} but has not been demonstrated in prior work.
    Abbreviations: diff = differentiable, opti = optimization.} \\
    \end{tabular}}
    \label{table1}
\end{table*}

Many existing works \cite{Ames.etal.TAC17,Lindemann.Dimarogonas.LCSS19,Xiao.Belta.TAC22}, unify CBF and CLF within a QP framework by relaxing the CLF with a slack variable penalized in the cost function to improve feasibility.

\subsection{Collision Detection via Minkowski Operation} \label{subsec:prelim_minkowski}

The workspace $\mathcal{W}$ of a robot is the physical (2D or 3D Euclidean) environment in which it operates, consisting of the space where the robot can move and the obstacles it must avoid. The configuration space ($\mathcal{C}$-space) represents all possible positions, orientations, etc, with each point defining a unique configuration. In contrast, the \mdspace~is
a dual of 
$\mathcal{W}$, and
is a Euclidean space with the same (2D or 3D) dimensionality as  $\mathcal{W}$. In  \mdspace, the robot is shrunk to a point at the origin, and the obstacles are transformed to encode the robot’s geometry and configuration (see Fig.~\ref{fig: sd_CSO}).

Minkowski operations map problems defined in  $\mathcal{W}$ to \mdspace, where distance computation (and thus collision detection) can be formulated
directly \cite{Cameron.Culley.ICRA86,gjk,epa,Ericson}. 
Given two sets $\mathcal{A}, \mathcal{B}\subset\mathbb{R}^n$, their Minkowski sum is defined as the set of all pairwise sums of points taken from each set:
    \begin{align}
        \mathcal{A}\oplus\mathcal{B}=\{\bm{a}+\bm{b}\mid \bm{a}\in \mathcal{A}, \bm{b}\in\mathcal{B}\},
        \label{eq: mink_sum}
    \end{align}
where $\bm{a},\bm{b}$ are position vectors in sets of $\mathcal{A},\mathcal{B}$, respectively.
\begin{defi} [\it Configuration Obstacle (\cobs) \cite{LaValle.Planning.06}] 
\normalfont
    Given a robot $\mathcal{R}$ and an obstacle $\mathcal{O}$, the corresponding 
    \cobs~is:
    \begin{align}
        \mathcal{C}_\mathcal{R}(\mathcal{O}) = \mathcal{O}\oplus (-\mathcal{R})\label{eq: cso},
    \end{align}
    where $-\mathcal{R}=\{-\bm{p}\mid \bm{p} \in \mathcal{R}\}$ \textcolor{black}{is the reflection of $\mathcal{R}$ through the origin, where $\bm{p}$ is a position vector of $\mathcal{R}$. Note that we refer to the Minkowski sum of $\mathcal{O}$ and $-\mathcal{R}$ as the \textit{Minkowski difference}, following collision-detection convention.}\footnote{In fields like mathematics and image processing, the Minkowski difference has a different (alternative) definition, where $\mathcal{A}\ominus \mathcal{B}$ is the set that, when added to $\mathcal{B}$, recovers $\mathcal{A}$.}
\end{defi}

\begin{lemma}[Collision between Convex Sets \cite{gjk}]
    A robot $\mathcal{R}$ and an obstacle $\mathcal{O}$ intersect (are in collision) if and only if their \cobs, $\mathcal{C}_\mathcal{R}(\mathcal{O})$, contains the origin $\bm{0}_c$.\label{lem: collision}
\end{lemma}
\begin{remark}
   While the Minkowski sum \eqref{eq: mink_sum} and \cobs~computation \eqref{eq: cso} are general operations that apply to any shape, it is worth emphasizing that
   Lem.~\ref{lem: collision} holds only for convex sets. \textcolor{black}{For well-behaved non-convex sets (e.g., non-fractal) that can be decomposed into convex subsets, a CO is computed for each convex-robot-convex-obstacle pair, and Lem. 1 applied iteratively. In this paper, we consider nonempty convex polygons for both robot and obstacle.}
\end{remark}

\begin{lemma}[Properties of the \cobs\cite{4M}]
    Given 
    convex, nonempty polygons $\mathcal{R}$ and $\mathcal{O}$ with $\ell_r$ and $\ell_o$ edges. Their 
    \cobs, $\mathcal{O}\oplus(-\mathcal{R})$, is a convex, nonempty polygon with 
    $\leq \ell_r+\ell_o$ edges, and can be computed in $O(\ell_r+\ell_o)$ time.
    \label{lem: properties}
\end{lemma}
By Lem.~\ref{lem: collision} and Lem.~\ref{lem: properties}, the collision detection between two convex polygons reduces to checking whether the origin lies inside their \cobs, as illustrated in Fig.~\ref{fig: sd_CSO}.
\subsection{Signed distance for Collision Avoidance} \label{subsec:prelim_sd}
While it may be sufficient to use the minimum distance to enforce safety in some scenarios, this measure alone is incapable of addressing situations where the robot is already in contact with an obstacle. To fully characterize the spatial relationship between two sets---whether disjoint or intersecting---the notion of signed distance\footnote{Similar notion appears as the minimal translational distance in \cite{Cameron.Culley.ICRA86}.} \cite{Schulman.etal.IJRR14} is required. 
\begin{defi} [\it Signed Distance via \cobs] 
\normalfont
    The signed distance between a robot $\mathcal{R}$ and an obstacle $\mathcal{O}$ is defined as
    \begin{align}\begin{split}
        \text{sd}(\mathcal{R},\mathcal{O}) &=\text{dist}(\mathcal{R},\mathcal{O})-\text{pd}(\mathcal{R},\mathcal{O})\label{eq: sd},
        \end{split}
    \end{align}
    where $\text{dist}(\cdot,\cdot)$ and $\text{pd}(\cdot,\cdot)$ are the minimum distance and penetration depth, and defined \cite{Zhang.OBCA.TCST20} as
    \begin{align}\begin{split}
        \text{dist}(\mathcal{R},\mathcal{O})
        &=\min\{\|\bm{t}\|\mid (\mathcal{R}+\bm{t})\cap\mathcal{O}\neq\varnothing\}
        \\
        &=\min\{\|\bm{t}\|\mid (\bm{0}_c+\bm{t})\cap\mathcal{C}_\mathcal{R}(\mathcal{O})\neq\varnothing\},\\[0.5em]
        \text{pd}(\mathcal{R},\mathcal{O})
         &=\min\{\|\bm{t}\|\mid (\mathcal{R}+\bm{t})\cap\mathcal{O}=\varnothing\}\\
        &=\min\{\|\bm{t}\|\mid (\bm{0}_c+\bm{t})\cap\mathcal{C}_\mathcal{R}(\mathcal{O})=\varnothing\}.
        \label{eq: sd_sub}
    \end{split}\end{align}
\end{defi}
\noindent 
The signed distance is positive when $\mathcal{R}$ and $\mathcal{O}$ are disjoint and negative when they intersect.
Thus, collision avoidance can be enforced by ensuring $\text{sd}(\mathcal{R},\mathcal{O})>0$. The signed distance between two sets can equivalently be expressed as the signed distance between the origin $\bm{0}_c$ and their \cobs. 
\removeForSpace{%
Prior work \cite{Singletary.etal.RAL22,Wu.etal.optfree.arxiv25} uses the SDF as CBFs with different approximation methods to circumvent the difficulty of computing the derivative of the implicit SDF, at the cost of generating conservative maneuvers.%
} 
\textcolor{black}{The closest related work \cite{Thirugnanam.etal.ACC22} constructs CBFs based on implicit MDFs by posing a convex program with dual variables, and enforces the resulting NCBF with a linear program (LP) that imposes a lower bound on the CBF derivative.} Moreover, \cite{Thirugnanam.etal.ACC22} considers only minimum distance and not penetration depth (cases where the system starts in collision). In contrast, we propose a new framework for SDF-based CBFs that incorporates penetration depth, as detailed in Sec.~\ref{sec: method}.

\section{Problem Formulation} \label{subsec:prelim_formulation}
The robot's state is denoted $\bm{x}$. Both robot and obstacles are defined using the halfspace-representation (H-rep), $\mathcal{R}(\bm{x}):= \{y\in \mathbb{R}^{d}\mid A_r(\bm{x}) y\leq b_r(\bm{x})\}$ and $\mathcal{O}_i:= \{y\in \mathbb{R}^{d}\mid A_{o_i} y\leq b_{o_i}\}$, with $A_r\in\mathbb{R}^{\ell_r\times d},~b_r\in\mathbb{R}^{\ell_r}$,  $A_{o_i}\in\mathbb{R}^{\ell_{o_i}\times d},~b_{o_i}\in\mathbb{R}^{\ell_{o_i}}$, and $d\in\{2,3\}$ for 2D and 3D spaces. $\ell_r$ and $\ell_{o_i}$ denote the number of edges of the robot and the $i$-th obstacle, respectively. The corresponding \cobs s are denoted as $\mathcal{O}_i^\mathcal{C}(\bm{x})=\{y\in \mathbb{R}^{d}\mid A^\mathcal{C}_i(\bm{x}) y\leq b^\mathcal{C}_i(\bm{x})\}$, omitting the subscript $\mathcal{R}$ for brevity, where $A_i^\mathcal{C}\in\mathbb{R}^{\ell_{\mathcal{C}_i}\times d},~b_i^\mathcal{C}\in\mathbb{R}^{\ell_{\mathcal{C}_i}}$, and $\ell_{\mathcal{C}_i}$ denotes the number of edges of the $i$-th \cobs.

By abuse of notation, we write ${\mathcal{R}=\mathcal{R}(\bm{x})}$ and $\mathcal{O}_i^\mathcal{C}=\mathcal{O}_i^\mathcal{C}(\bm{x})$. 
The \cobs~is specific to the obstacle and the robot's current configuration, and thus 
depends on $\bm{x}$. In our framework, the \cobs~is used to compute either the minimum distance or the penetration depth, depending on whether the origin lies outside or inside the \cobs, respectively.

We make the following assumptions:
\begin{enumerate}[label=\arabic*.]
    \item The robot $\mathcal{R}$ is a rigid body, is uniformly bounded, and has nonempty interior, $\forall \bm{x}\in \mathcal{X}$.\label{item:assp1}
    \item $A^\mathcal{C}_i(\bm{x}), b^\mathcal{C}_i(\bm{x})$ are 
    \textcolor{black}{differentiable a.e.},
    and the inequality constraints provide a minimal representation of the \cobs. For almost all $\bm{x}\in \mathcal{X}$, $\ell_{\mathcal{C}_i}=\ell_r+\ell_{o_i}$. Note that, if $\mathcal{R}$ and $\mathcal{O}_i$ have parallel edges, the associated constraints $A^\mathcal{C}_i,~b^\mathcal{C}_i$ may lose one/more rows, depending on the number of parallel edges. Yet, such degeneracy occurs only on a measure zero set and thus does not affect the optimization for almost all $\bm{x}$ (as reported in \cite{Schulman.etal.IJRR14}).\label{item:assp2}
    \item All obstacles are static, uniformly bounded, and have nonempty interior. Also, the information about obstacles is accessible to the robot.\label{item:assp3}
\end{enumerate}
We define the problem using the robot $\mathcal{R}$ and a single obstacle $\mathcal{O}$, omitting the subscript $i$ for simplicity. 
\begin{prob}
    Given a robot $\mathcal{R}$ with known affine dynamics in the form of \eqref{eq: ctrl-affine} and a static obstacle $\mathcal{O}$, find a state-feedback optimal controller $\bm{u}^*(\bm{x})$ that both
    \begin{enumerate}
        \item \textit{achieves objectives:} minimizes energy consumption $J=\int_{0}^{T}\bm{u}^{\top}\bm{u} dt$, and ensures the robot reaches the desired position $(x_d,y_d)$, 
        \item \textit{satisfies constraints:} safety ---  requiring $\text{sd}(\mathcal{R},\mathcal{O})>d_{\text{safe}}$, where $d_{\text{safe}}\geq 0$ is the user-defined safety margin, and actuator limitations $\bm{u}_{\min}\leq\bm{u}\leq\bm{u}_{\max}$.
    \end{enumerate} 
    \label{prob}
\end{prob}

\begin{remark}
    This formulation extends to multiple obstacles by applying the same safety constraint independently to each pair $(\mathcal{R},\mathcal{O}_i)$. Additionally, non-convex polygons can be represented as a union of multiple convex polygons.
    \label{rem2}
\end{remark}
We now develop a framework for constructing CBFs based on the SDF without introducing geometric approximations. 
\removeForSpace{%
We formulate both minimum distance and penetration depth problems as convex optimizations in \mdspace. These optimizations output a critical point, and its derivative, computed through differentiable optimization, is used to obtain the
derivative of CBFs. The complete proposed framework, including all components introduced above, is described in the next section.
} 

\section{SDF-Based CBFs for Polytopes in \mdspace}\label{sec: method}
This section presents our main result of constructing CBFs based on the ``exact'' SDF by recasting its computation as convex optimization in \mdspace.
The derivative of the SDF 
is then computed using differentiable convex optimization.

\subsection{Optimization-Based SDF Formulation in \mdspace
}\label{sec: 5a-obsd}
The minimum distance between $\mathcal{R}$ and $\mathcal{O}$ when the two shapes are disjoint ($\mathcal{R}\cap \mathcal{O} = \varnothing$) can be computed by solving a QP in the workspace:
\begin{align}\begin{split}
    &\min\,\,\|\mathbf{x}-\mathbf{y}\|^2_2
   \quad\text{subject to}\,\,\,\, A_r \mathbf{x}\leq b_r, \,\,A_o\mathbf{y} \leq b_o\\
   &\,\,\text{with}\,\,\,\, \text{dist}(\mathcal{R},\mathcal{O}) =\|\mathbf{x}^*-\mathbf{y}^*\|_2,
   \label{eq: dist_w}
\end{split}\end{align}
where $\|\cdot\|_2$ is the Euclidean norm and $(\mathbf{x}^*,\mathbf{y}^*)$ is the optimal solution pair to \eqref{eq: dist_w}. The distance between sets can be reformulated as the distance between a point and a set \cite{Boyd.Vandenberghe.04}, 
$    
\text{dist}(\mathcal{R},\mathcal{O})=\text{dist}(\bm{0}_c,\mathcal{O}\oplus(-\mathcal{R})),
$    
where $\mathcal{O}\oplus(-\mathcal{R})$ is their \cobs.
\noindent We formulate the equivalent problem in \mdspace~as 
\begin{align}\begin{split}
    &\min\,\,\|\bm{z}\|^2_2\quad\text{subject to}\,\,\,\,A^\mathcal{C} \bm{z}\leq b^\mathcal{C}\\
   &\,\,\text{with}\,\,\,\, 
   \text{dist}(\bm{0}_c,\mathcal{O}^\mathcal{C}) =\|\bm{z}^*\|_2,
   \label{eq: dist_c}
\end{split}\end{align}
where $\bm{z}^*$ is the optimal solution of \eqref{eq: dist_c} and it can also be interpreted as the projection\footnote{The projection of a point $\bm{x}$ on a set $C$ is any point $\bm{z} \in C$ that is closest to $\bm{z}$, i.e., satisfies $\| \bm{z}-\bm{x} \| = \text{dist}(\bm{x}, C)$.} of the origin $\bm{0}_c$ on $\mathcal{O}^\mathcal{C}$.
Similarly, the penetration depth $\text{pd}(\mathcal{R}, \mathcal{O})$ can be interpreted as the problem of finding the minimum translation vector (MTV) that separates two intersecting convex sets. It can be conceptually formulated as an optimization problem in the workspace (a similar formulation appears in \cite{Gao.arXiv.2023}) as
\begin{align}\begin{split}
    &\text{pd}(\mathcal{R},\mathcal{O}) =\,\, \min\,\, \|\bm{z}\|_2\quad \text{subject to}\,\,\,\, \mathcal{R}'\cap\mathcal{O} = \varnothing,
    \label{eq: pd_w}
\end{split}\end{align}
where $\bm{z}\in\mathbb{R}^d$ is the MTV, and $\mathcal{R}'= \{y\in \mathbb{R}^{d}\mid A_r y\leq b_r+A_r\bm{z}\}$ is the translation of $\mathcal{R}$ by $\bm{z}$.
However, unlike in the minimum distance case, it is unclear how to represent the separation constraint $\mathcal{R}'\cap \mathcal{O} = \varnothing$ as a convex constraint. 

To address this issue, we propose a reformulation of \eqref{eq: pd_w} as LP in \mdspace~by leveraging the notion of the depth\footnote{The depth of a point $\bm{x}$ on a convex set $C$ is the radius of the largest ball, centered at $\bm{x}$, that lies entirely within $C$.} of a point relative to a convex set \cite{Boyd.Vandenberghe.04} 
\begin{align}\begin{split}
    & \text{depth}(\bm{0}_c,\mathcal{O}^\mathcal{C}) =\,\,{\max}\,\, s
     \\
   &\,\,\text{subject to}\,\,\,\,s\geq 0,\,\,s\|A^\mathcal{C}_k\|_2  \leq b^\mathcal{C}_k,\,\, k=1,\ldots,\ell_{\mathcal{C}},
   \label{eq: pd_c}
\end{split}\end{align}
where $s\in\mathbb{R}$ is the minimal translational distance required to separate two overlapping sets. 
The vector $A^\mathcal{C}_k$ and scalar $b^\mathcal{C}_k$, denoting the $k$-th row $A^\mathcal{C}$ and $b^\mathcal{C}$, respectively, define the bounding hyperplanes of $\mathcal{O}^\mathcal{C}$. It is shown in \cite{Cameron.Culley.ICRA86,epa} that $\text{pd}(\mathcal{R},\mathcal{O})=\text{depth}(\bm{0}_c,\mathcal{O}^\mathcal{C})$.

However, while \eqref{eq: pd_c} gives the penetration depth (scalar), we also need the direction of separation to recover the optimal solution $\bm{z}^*$ of \eqref{eq: pd_w}, which cannot be solved directly. 
This direction can be obtained by identifying the active constraint from the dual variable of \eqref{eq: pd_c}.
Let $k_\text{act}$ be the set of active indices, and $A^{\mathcal{C}}_{k_\text{act}}$ be the matrix containing the rows corresponding to the active constraints at the optimal solution $s^*$. By strong duality in LP, complementary slackness holds and implies that nonzero dual variables $\lambda_{k_\text{act}}$ correspond to the active constraints. 

In most cases, \textcolor{black}{a single constraint is active, so}
$A^{\mathcal{C}}_{k_\text{act}}$ reduces to one row vector, denoted as $a^{\mathcal{C}}_{k_\text{act}}$, which uniquely determines the separation direction. When multiple constraints are active, $A^{\mathcal{C}}_{k_\text{act}}$ becomes a matrix. \textcolor{black}{As the origin is equidistant from multiple active planes, the separating direction lies in the convex cone of their normals.}
The presence of multiple active constraints occurs only for a measure-zero set of configurations.
\textcolor{black}{Moreover, if such a case occurs in practice, any active constraint yields a valid separating direction.
}
Once the active constraint of \eqref{eq: pd_c} is identified, corresponding to the hyperplane $\mathcal{H}=\{y\in \mathbb{R}^{d}\mid a^{\mathcal{C}}_{k_\text{act}}y=b^{\mathcal{C}}_{k_\text{act}}\}$, $\bm{z}^*$ can be computed by projecting the origin $\bm{0}_c$ onto $\mathcal{H}$ \cite{Boyd.Vandenberghe.04} as
\begin{align}
    \bm{z}^* = \text{proj}_{\mathcal{H}}(\bm{0}_c) = b^{\mathcal{C}}_{k_\text{act}}(a^{\mathcal{C}}_{k_\text{act}})^{\top}
    /\|a^{\mathcal{C}}_{k_\text{act}}\|^2_2.
    \label{eq: pd_zstar}
\end{align}

The signed distance \eqref{eq: sd} is computed by solving either the QP in \eqref{eq: dist_c} or the LP in \eqref{eq: pd_c}, with feasibility guaranteed by the assumption \ref{item:assp1} and Lem.~\ref{lem: properties}. The equivalence of computing the signed distance in both $\mathcal{W}$- and \mdspace s is shown in Fig.~\ref{fig: sd_CSO}. The key advantages of formulating the optimization problems in \mdspace\,are: (1) it enables convex optimization (LP) formulation for penetration depth computation; and (2) the minimum distance and penetration depth can be viewed as companion problems that provide a generalized measure of distance. This allows us to compute derivative of the signed distance using the same differentiable optimization framework, see Sec.~\ref{sec: 5c-dh}. 

\subsection{CBF Construction Using Signed Distance}
The point $\bm{z}^*$, defined in \mdspace ~and obtained from our optimization-based signed distance formulation, is referred to as the critical point, as its derivative is used in computing the derivative of CBF. Geometrically, $\bm{z}^*$ lies on the boundary of the \cobs, $\partial\mathcal{O}^\mathcal{C}$, and is the closest point to the origin $\bm{0}_c$ in \mdspace. Thus, the signed distance is equal to $\pm\|\bm{z}^*\|_2$. 

Based on this, we propose the following CBF candidate: 
\begin{align}\begin{split}   
    &h(\bm{x})=\text{sd}(\mathcal{R}(\bm{x}),\mathcal{O})-d_\text{safe}=\text{sd}(\bm{0}_c,\mathcal{O}^\mathcal{C}(\bm{x}))-d_\text{safe},
    \\[0.3em]
    &\,\,\text{where}\,\,
    \text{sd}(\mathcal{R}(\bm{x}),\mathcal{O})
    =\left\{\begin{matrix}
\|\bm{z}^*(\bm{x})\|_2 \,\,\,\,\,\quad \text{if}\,\,\, 
\bm{0}_c\notin\mathcal{O}^\mathcal{C}\\
-\|\bm{z}^*(\bm{x})\|_2\quad \text{if}\,\,\,
\bm{0}_c\in\mathcal{O}^\mathcal{C}
\end{matrix}\right.,
    \label{eq: cbf}
\end{split}
\end{align}
where $\bm{z}^*(\bm{x})\in\mathbb{R}^d$ is the optimal solution to \eqref{eq: dist_c} if $\mathcal{R}\cap \mathcal{O}=\varnothing$, and can be indirectly found by solving \eqref{eq: pd_c} with its dual variable analysis in the case of $\mathcal{R}\cap \mathcal{O}\neq\varnothing$. Recall that both the critical point $\bm{z}^*(\bm{x})$ and the \cobs, $\mathcal{O}^\mathcal{C}(\bm{x})$, depend on the robot's state $\bm{x}$. The associated safe set of the system is defined as $C = \{\bm{x}\in\mathbb{R}^n:\text{sd}(\bm{0}_c,\mathcal{O}^\mathcal{C}(\bm{x}))\geq d_\text{safe} \}.$
Note that the SDF-based CBF \eqref{eq: cbf} is implicit and is obtained from an optimization that depends on system state $\bm{x}$. It is important to note that the SDF in Euclidean space is differentiable a.e. \cite{Singletary.etal.RAL22}, which means it is differentiable at all points except for a set of (Lebesgue) measure zero.
\begin{lemma}
    \textcolor{black}{Let $\mathcal{X}_f=\{x\in\mathcal{X}\mid\mathcal{R}(\bm{x}) \cap \mathcal{O}=\varnothing\}$. The optimal solution $\bm{z}^*(\bm{x})$ of \eqref{eq: dist_c} is unique for every $\bm{x}\in\mathcal{X}_f$, and $\bm{z}^*(\cdot)$ is continuous and differentiable a.e. on $\mathcal{X}_f$.}
\end{lemma}
\begin{proof}
By Lem. 2, the \cobs~set $\mathcal{O}^\mathcal{C}(\bm{x})$ is nonempty, closed, and convex, and therefore, the Euclidean projection of the origin $\bm{0}_c$ on the set $\mathcal{O}^\mathcal{C}(\bm{x})$ is unique \cite[Chap.~8]{Boyd.Vandenberghe.04}.
As \eqref{eq: dist_c} is a strictly convex QP, its solution $\bm{z}^*(\bm{x})$ is continuous and differentiable a.e. based on \cite[Thm.~1]{Amos.Kolter.ICML17}.
\end{proof}
\begin{remark}
    \textcolor{black}{Extending Lemma 3 to the penetration depth case is deferred to future work. The min-distance case is a strictly convex QP, while penetration depth is an LP, so the differentiability results of \cite{Amos.Kolter.ICML17} do not directly apply. 
    The min-distance and penetration depth cases are complementary, and we focus on the min-distance case in the next subsection.
    While our LP formulation performs well empirically and yields stable derivative estimates under the same \mdspace~framework, a rigorous differentiability analysis of the penetration depth case is planned for future work.}
\end{remark}

\subsection{Derivative Calculation Using Differentiable Optimization}\label{sec: 5c-dh}
To enforce the CBF constraint on system dynamics, we need to find the derivative of $h(\bm{x})$. However, the CBF $h(\bm{x})$ \eqref{eq: cbf} depends on the solution to an optimization problem $\bm{z}(\bm{x})$, defined implicitly to state $\bm{x}$. Therefore, we cannot differentiate it w.r.t. $\bm{x}$ directly. Given this implicit dependence, we first apply the chain rule to express $\partial h / \partial \bm{x}$ 
as:
\begin{align}\begin{split} 
        \frac{\partial h}{\partial \bm{x}} &=\frac{\partial h}{\partial \bm{z}^*}\frac{\partial \bm{z}^*}{\partial \bm{x}},
    \label{eq: dhdx}
\end{split}\end{align}
where \textcolor{black}{$\partial h/\partial \bm{z}^* = \bm{z}{^*}^{\top}/(h+d_\text{safe})\in\mathbb{R}^{1\times d}$.}
The second term $\partial \bm{z}^*/\partial \bm{x}\in\mathbb{R}^{d\times n}$
is the Jacobian of the optimal solution $\bm{z}^*$ to \eqref{eq: dist_c}, which describes how it changes as inequality constraints of CO vary with the state $\bm{x}$. 
\begin{corollary} Let \textcolor{black}{$\bm{z}^*(\bm{x})$} be the optimal solution to the QP \eqref{eq: dist_c} \textcolor{black}{for $\bm{x}\in\mathcal{X}_f$}. Its Clarke generalized subdifferential \cite{Clarke.1975} $\partial\bm{z}^*(\bm{x})$ is nonempty \textcolor{black}{for all $\bm{x}\in\mathcal{X}_f$}, and has a single unique Jacobian for all but a measure zero set of points.
\end{corollary}
\begin{proof}
    This follows directly from \cite[Thm.~1]{Amos.Kolter.ICML17}, as the objective function of QP \eqref{eq: dist_c} is strictly convex.
\end{proof}
Since \eqref{eq: dist_c} is a QP in standard form, the Jacobian can be computed using the framework introduced in \cite{Amos.Kolter.ICML17}. We start by finding the Karush-Kuhn-Tucker (KKT) conditions of \eqref{eq: dist_c}, where the stationarity and complementary slackness conditions can be compactly expressed in matrix form as:
\begin{align*}\begin{split}
    G(\bm{z},\boldsymbol{\lambda}) = \begin{bmatrix}
    2I_{d}\bm{z}+({A^\mathcal{C}})^{\top}\boldsymbol{\lambda}
     \\
    D(\boldsymbol{\lambda})(A^\mathcal{C}\bm{z}-b^\mathcal{C})
    \end{bmatrix},
\end{split}\end{align*}
where $I_d$ is the $d$-dimensional identity matrix, $\boldsymbol{\lambda}\in\mathbb{R}^{\ell_\mathcal{C}}$ is the dual variable of \eqref{eq: dist_c}, and $D(\cdot)$ creates a diagonal matrix from a vector. Let $\bm{\xi}:=(\bm{z},\boldsymbol{\lambda})$. Thus, a convex optimization solver can be viewed as a root-finding method that seeks the optimal solution $\bm{\xi}^*(\bm{x})$
by solving $G(\bm{\xi}^*)=0$
(suppressing the dependency of $\bm{\xi}$ on $\bm{x}$ for brevity).
The partial derivative of $\bm{z}^*$ w.r.t. $\bm{x}$ can be computed by differentiating the KKT conditions using the implicit function theorem \cite{Dontchev.Rockafellar.09}:
\begin{align*}\begin{split}
    D_{\bm{x}} G(\bm{\xi}^*) = 0\,\,
    \Rightarrow \,\,\partial_{\bm{\xi}}G(\bm{\xi}^*)\,\,\partial_{{\bm{x}}}(\bm{\xi}^*)+\partial_{\bm{x}}G(\bm{\xi}^*)=0.
\end{split}\end{align*}
Rearranging the matrix equation yields 
\begin{align}
\partial_{\bm{x}}(\bm{\xi}^*) &= -\left[ \partial_{\bm{\xi}}G(\bm{\xi}^*)\right]^{-1} \partial_{\bm{x}}G(\bm{\xi}^*), \label{eq: dzdx} \\[0.5em]
\begin{split}
    \partial_{\bm{\xi}}G(\bm{\xi}^*) &=
    \begin{bmatrix}
    2I_{d} & ({A^\mathcal{C}})^{\top} \\
    D(\boldsymbol{\lambda}^*)A^\mathcal{C} & D(A^\mathcal{C}\bm{z}^*-b^\mathcal{C})
    \end{bmatrix}, \\[1ex]
    \partial_{\bm{x}}G(\bm{\xi}^*) &=
    \begin{bmatrix}
    d_{\bm{x}}({A^\mathcal{C}})^{\top} \boldsymbol{\lambda}^* \\
    D(\boldsymbol{\lambda}^*)\left(d_{\bm{x}}A^\mathcal{C} \bm{z}^*-d_{\bm{x}} b^\mathcal{C} \right)
    \end{bmatrix},
    \end{split}\label{eq: dG}
\end{align}
where $d_{\bm{x}}A^\mathcal{C}\in\mathbb{R}^{\ell_\mathcal{C}\times d\times n}$ is a 3D tensor and $d_{\bm{x}}b^\mathcal{C}\in\mathbb{R}^{\ell_\mathcal{C}\times n}$ represents how the bounding constraints of \cobs, $A^\mathcal{C},b^\mathcal{C}$, changes w.r.t. state $\bm{x}$. 
Note that  $\bm{z}^*,\boldsymbol{\lambda}^*,A^\mathcal{C},b^\mathcal{C}$ all depend on the state $\bm{x}$, which we denote compactly as $\bm{z}^*(\bm{x}),\boldsymbol{\lambda}^*(\bm{x}),A^\mathcal{C}(\bm{x}),b^\mathcal{C}(\bm{x})$. The Jacobian of the optimal solution $\bm{z}^*$ can be extracted from the associated block \eqref{eq: dzdx}, which also provides the derivatives of dual variables. The partial derivative of $h$ can be computed using 
\eqref{eq: dhdx} to \labelcref{eq: dG}.
\begin{remark} 
As \eqref{eq: dG} is 
derived by differentiating the KKT conditions; complementary slackness holds. Thus, it is sufficient to consider only the active constraints when computing the derivatives. That is, we can replace $A^\mathcal{C},b^\mathcal{C}$ with $A^\mathcal{C}_{k_\text{act}},b^\mathcal{C}_{k_\text{act}}$.
\end{remark}

\subsection{Configuration Obstacle Computation}\label{sec: 5d-cso}
Consider a polygonal robot modeled as a rigid body in SE(2), with states 
(position and orientation) is $\bm{x}=[x,y,\theta]^{\top}$, and a nonzero initial configuration $\bm{x}_0=[x_0,y_0,\theta_0]^{\top}$. Its shape in H-rep can be parameterized by its configuration in $\mathcal{W}$-space as $\mathcal{R}(\bm{x})=\{y\mid A(\bm{x})y\leq b(\bm{x})\}$, where $A(\bm{x})=AR(\theta-\theta_0)^{\top},~b(\bm{x})= b+A(\bm{x})\bm{p}-A\bm{p}_0$ with $\bm{p}=[x,y]^{\top}$, $\bm{p}_0=[x_0,y_0]^{\top}$, and $A(\bm{x}_0),b(\bm{x}_0)$ defining the robot's shape at the initial configuration. $R(\theta) = [\cos\theta, -\sin\theta; \sin\theta, \cos\theta]$ is a rotation matrix.

However, this formulation does not extend directly to the \cobs~in \mdspace. In the special case of pure translation, the \cobs's shape remains unchanged, and its bounding hyperplanes translate in the opposite direction of the robot’s motion by definition \eqref{eq: cso}. Thus, the derivatives of the \cobs~bounding constraints w.r.t. position $\bm{p}$, have a closed-form expression: $d_{\bm{x}}A^\mathcal{C}=\bm{0}_{\ell_\mathcal{C}\times 2\times 2 },~d_{\bm{x}}b^\mathcal{C}=-A^\mathcal{C}$.

\textcolor{black}{In contrast, with rotation the bounding hyperplanes change both slope and order, causing nontrivial CO deformations. While the Minkowski sum can be computed via linear projection  \cite{Sadraddini.Tedrake.CDC19}, this method still does not provide an analytical expression for \cobs. Deriving an explicit relationship between the robot-obstacle representations in $\mathcal{W}$-space and the \cobs~in \mdspace~remains an open problem for future research.
}

\subsection{Controller Synthesis Using CLF-CBF-QP}
We implement \eqref{eq: CLF-CBF-QP} by discretizing the horizon \textcolor{black}{$[0,T]$ into equal intervals of $\Delta t$}, holding the state fixed at the start of each step. At time step $t=k\Delta t$, the control is held constant over $[t_k,t_{k+1})$, rendering the constraints affine in $\bm{u}$. We then solve the standard CLF-CBF-QP \cite{Ames.etal.TAC17} with $\bm{x}(t_k)$ to compute the optimal control $\bm{u}(t_k)=\bm{u}_k$, and apply $\bm{u}_k$ on $[t_k,t_{k+1})$ to propagate the dynamics \eqref{eq: ctrl-affine}.
\begin{align}\begin{split}
 \underset{\bm{u}(t),\delta(t)}{\min}
 \quad & \int_0^T\bm{u}(t)^{\top}\bm{u}(t)+p\delta(t)^2\\[0.2em]
 \text{subject to}\quad & L_fh(\bm{x})+ L_gh(\bm{x})\bm{u}+\gamma h(\bm{x})\geq \epsilon,\\ &L_fV(\bm{x})+ L_gV(\bm{x})\bm{u}+c V(\bm{x})\leq\delta,\\& \bm{u}_{\min}\leq\bm{u}\leq\bm{u}_{\max} ,\label{eq: CLF-CBF-QP}\end{split}\end{align}
where $\delta$ is a slack variable, used to relax the CLF constraint and $p > 0$ is a penalty. $\epsilon>0$ is a small constant to ensure the CBF will change sign when starting in the unsafe set in finite time. We refer the interested readers to \cite{10886196} for detailed discussion on finite-time convergence guarantee.

\section{Case Study and Implementation}
We demonstrate the proposed CBFs by solving \mbox{Problem \ref{prob}} in three case studies. We consider single-integrators (\ref{sec: case-single}) and unicycles that start in collision (\ref{sec: case-unicycle-collision}) or involve multiple obstacles (\ref{sec: case-multiple}). All simulations are conducted in MATLAB, using ‘quadprog’ to solve both signed distance and CLF-CBF-QP problems, and ‘ode45’ to integrate system dynamics at a sampling rate of 100 Hz, on a PC with an Intel i7-8700 6-core of 3.2 GHz CPU and 32 GB of RAM. 
(See video at \textnormal{\url{https://youtu.be/3Dh0gtDW8bE}} .)

\begin{figure}[b!]
	\captionsetup[subfigure]{labelformat=empty,aboveskip=-2.5pt,belowskip=-10pt}
	\centering
	\begin{subfigure}{\linewidth}
		\centering
		\includegraphics[width=.94\linewidth]{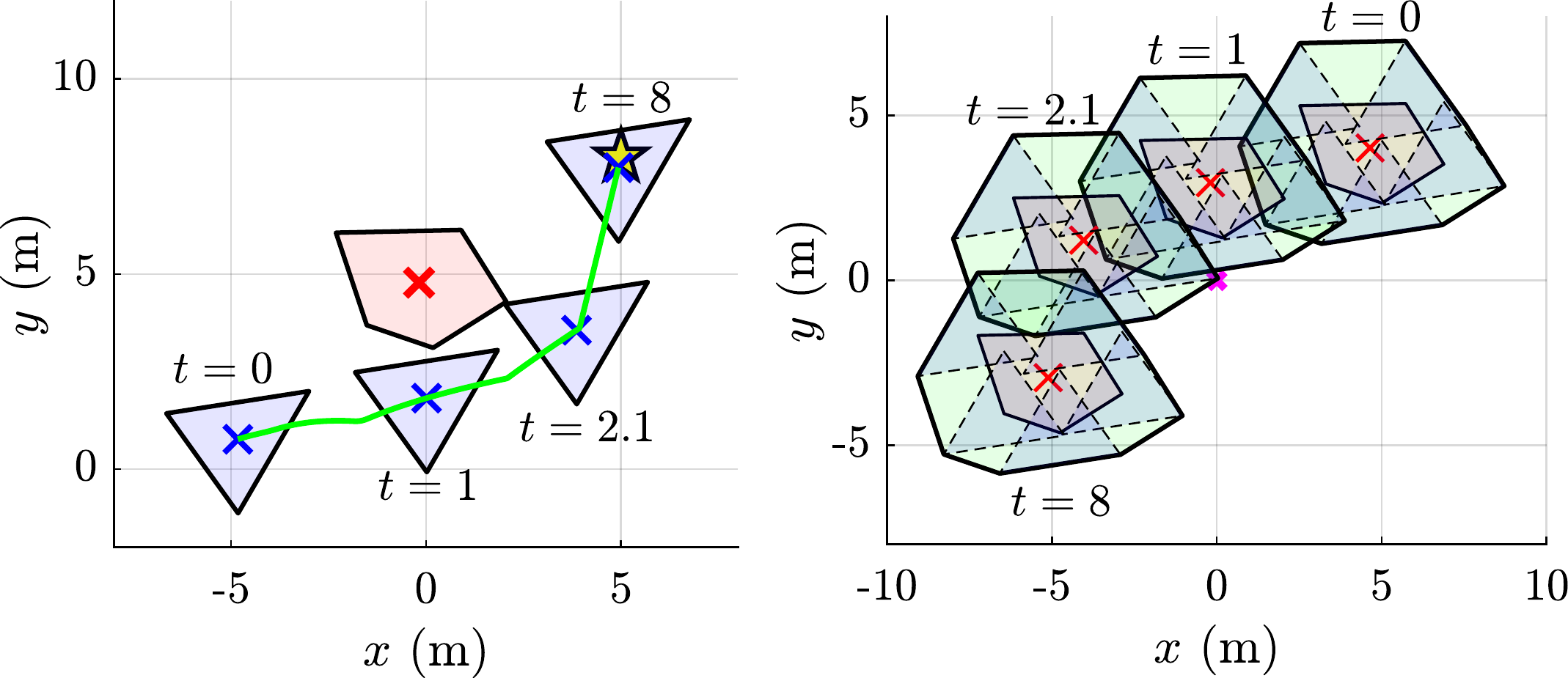}
            \caption{}
		\label{fig: case1-w-c}
	\end{subfigure}
        \\[0.3em]
        \begin{subfigure}{\linewidth}
		\centering
		\includegraphics[width=.94\linewidth]{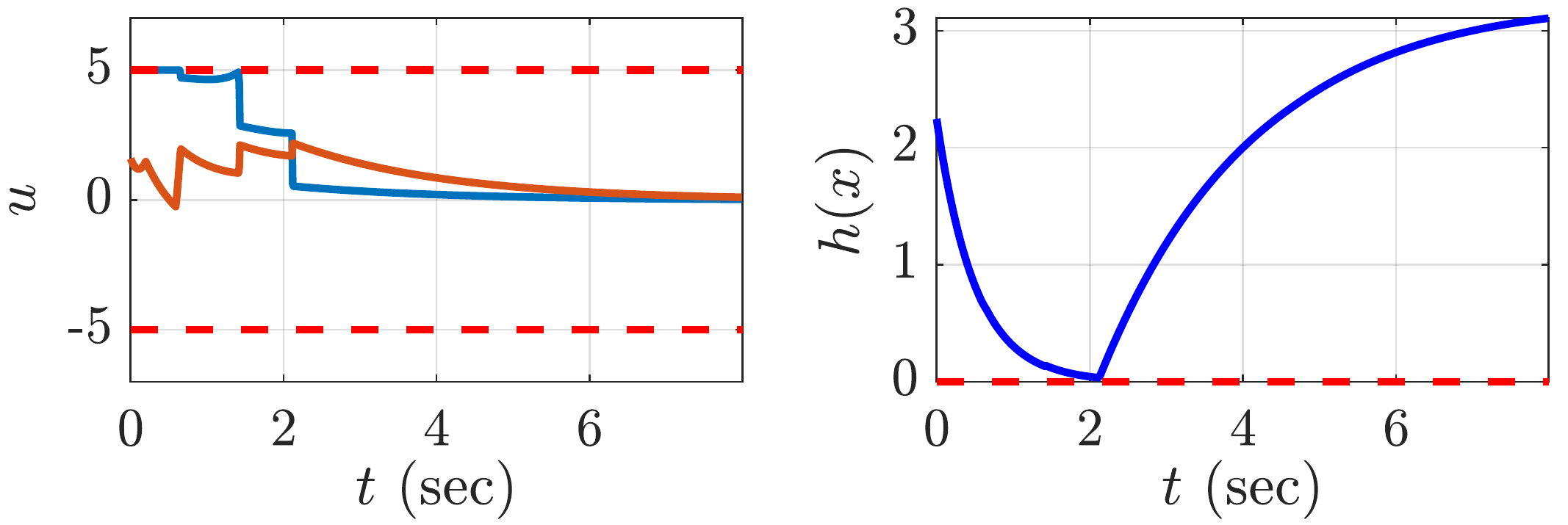}
		\caption{}
		\label{fig: case1-u-bx}
	\end{subfigure}
	\caption{\textbf{Top-Left:} The robot’s trajectory (green) moving left-to-right from the start (${t=0}$) to the goal (${t=8}$). Snapshots show the robot (blue) avoiding the obstacle (red). \textbf{Top-Right:} The corresponding motion of the \cobs~(green) moving right-to-left from ${t=0}$ to ${t=8}$ in \mdspace~ (see video \textnormal{\url{https://youtu.be/3Dh0gtDW8bE}}). \textbf{Bottom:} Control input $\bm{u}(t)$, and CBF $h(\bm{x})$ profiles. Both control inputs satisfy actuator limits and safety constraints.}\label{fig: case1}
\end{figure}

\subsection{Case I: Single-Integrator Model (Pure Translation)} \label{sec: case-single}
We first consider the scenario of a triangular robot $\mathcal{R}$ with single-integrator dynamics translating to the desired position $\bm{x}_d$ while avoiding a polygonal obstacle $\mathcal{O}$: $\bm{x} =[x,y]^{\top},\,\,\dot{\bm{x}}=\bm{u}=[u_1,u_2]^{\top}$, where $x,y$ denote the position, and $u_1,u_2$ are the control inputs (linear velocities). 
The dynamics can be rewritten in form of \eqref{eq: ctrl-affine} with $f(\bm{x})=\bm{0}_{2\times 1}$ and $g(\bm{x})=\bm{I}_2$. We use the proposed SDF-CBF \eqref{eq: cbf} to impose safety and a CLF $V(\bm{x})=(\bm{x}-\bm{x}_d)^2$ to reach $\bm{x}_d=(5,8)$. The other parameters are $\bm{x}(0)=(-4.83,0.77), \bm{u}_{\max}=-\bm{u}_{\min}=5$m/s, $d_\text{safe}=0, p=10, c=1, \gamma = 3$, $\epsilon=0$.
The derivative of the CBF is computed using \eqref{eq: dzdx} and \eqref{eq: dG}. Each control loop involves computing the \cobs~\eqref{eq: cso} and solving a QP for minimum distance \eqref{eq: dist_c} and CLF-CBF-QP \eqref{eq: CLF-CBF-QP}. Loops run at 100 Hz and take 7.38 ms on average. The simulation results (Fig.~\ref{fig: case1}) demonstrate the non-conservativeness of the proposed CBF, i.e., the robot safely moves in close proximity to the obstacle boundary.

\begin{figure*}[t!]
    \begin{minipage}{5.2in}
    \includegraphics[width=\textwidth]{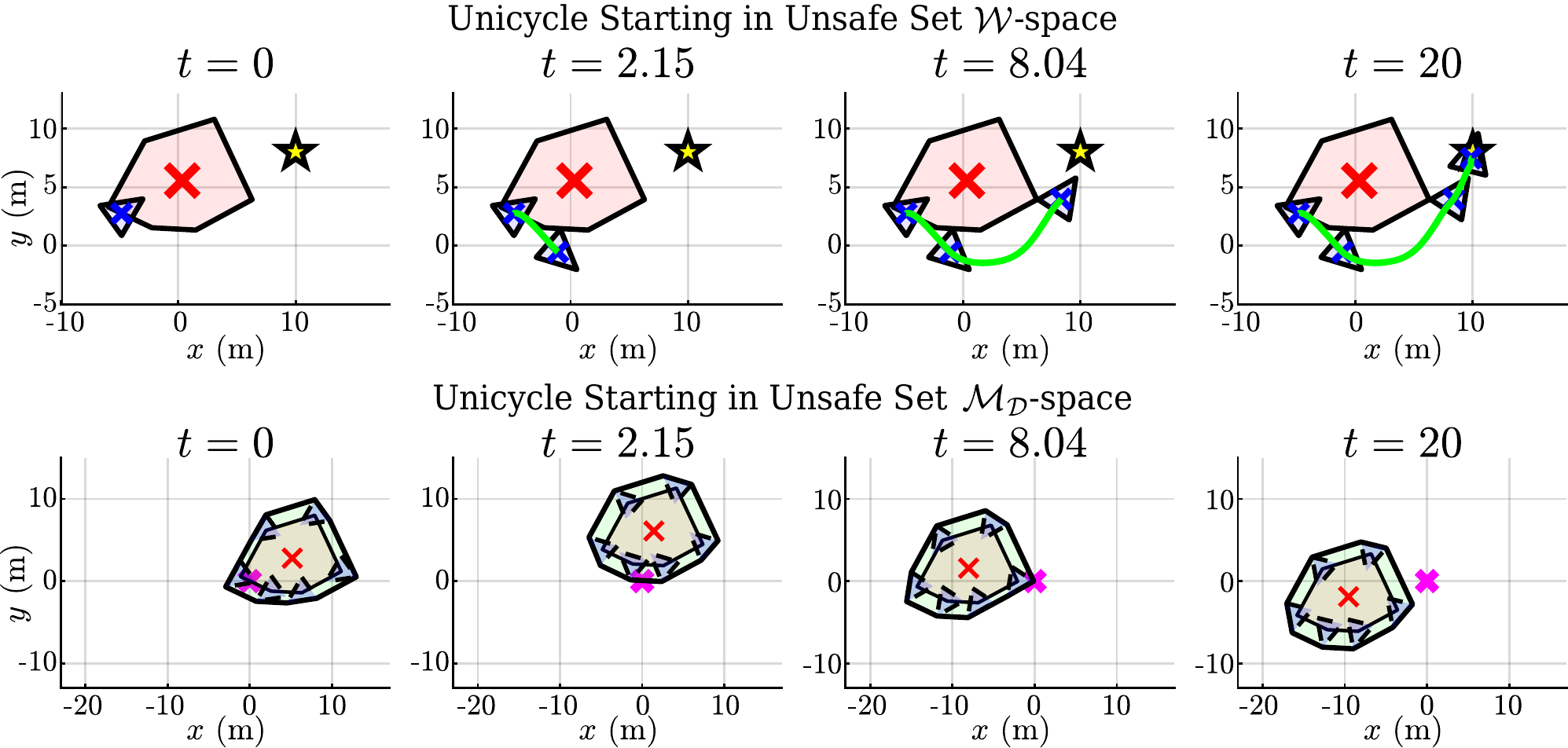}    
    \end{minipage}
    \begin{minipage}{1.7in}
    \includegraphics[width=.97\textwidth]{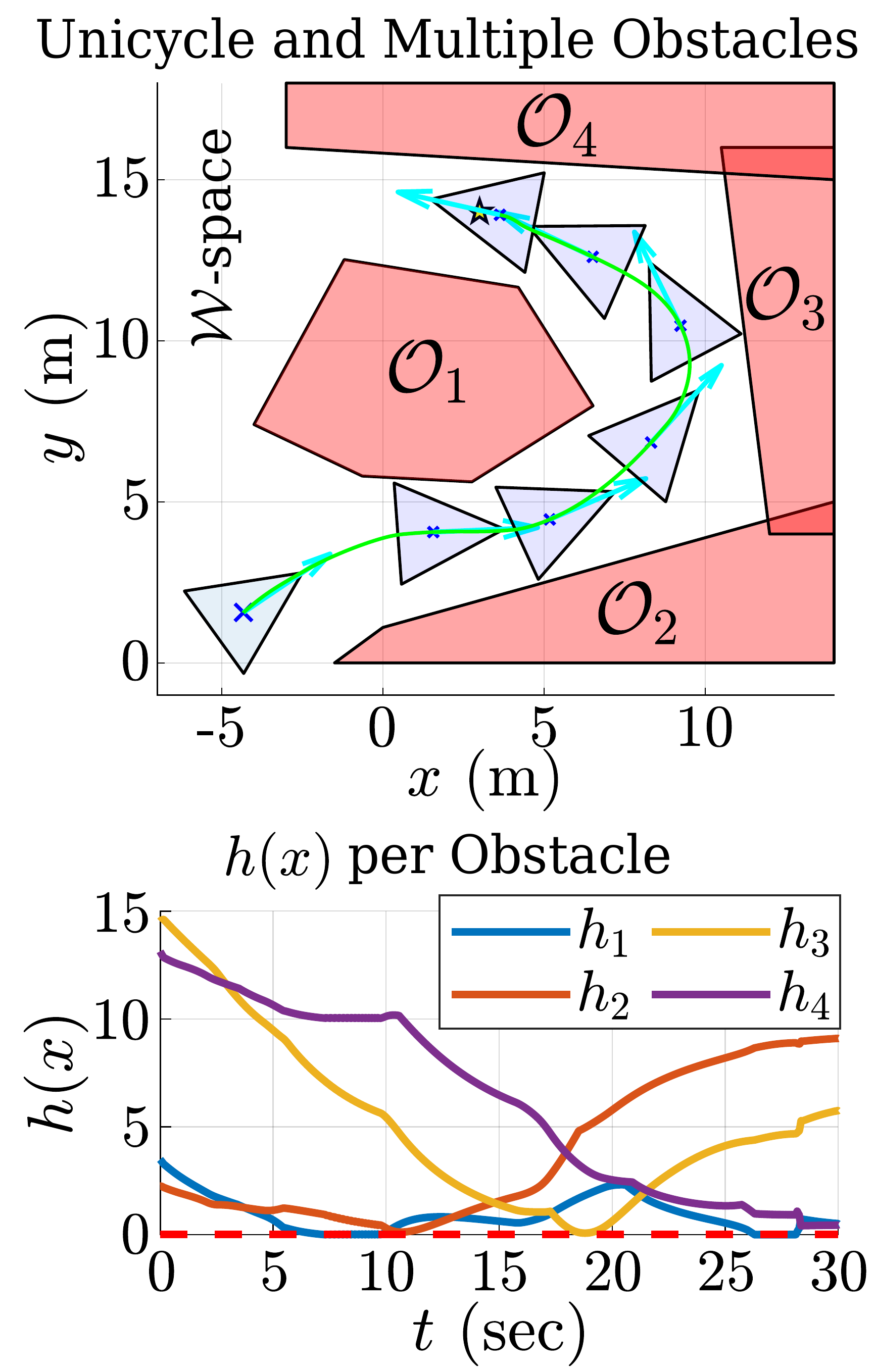}
    \end{minipage}
    \caption{\textbf{Left}: Four columns depict movement over time. At $t=0$ The robot starts in collision with the obstacle in $\mathcal{W}$-space (top), and the corresponding \cobs~contains the origin in \mdspace~ (bottom). As the robot moves out of the obstacle, the \cobs~moves away from the origin ($t=0$ to $t=2.15$ top and bottom). The collision is resolved by time $t=2.15$ seconds).
    \textbf{Right-Top:} Snapshots of robot motion in $\mathcal{W}$-space demonstrate the non-conservativeness of the proposed framework. \textbf{Right-Bottom:} The CBFs $h_i(\bm{x})$ for each obstacle stay above the safety margin over time, thereby guaranteeing safety.}
    \label{fig: unsafe_obstacle_combined}
\end{figure*}

\subsection{Case II: Unicycle Model}\label{sec: case-unicycle}
Now we consider the unicycle dynamics defined by ${\dot{x}=v\cos{\theta}}$, ${\dot{y}=v\sin{\theta}}$, ${\dot{\theta}=u_1}$, and ${\dot{v}=u_2}$,
where $x, y$ denote the position, $\theta$ is the orientation, $v$ is the linear speed, and $u_1, u_2$ are the two control inputs (turning rate and acceleration). The dynamics can be rewritten in form of \eqref{eq: ctrl-affine} with $\bm{x}=[x,y,\theta,v]^{\top}, \bm{u}=[u_1,u_2]^{\top}, f(\bm{x})=[v\cos{\theta},v\sin{\theta},0,0]^{\top}$ and $g(\bm{x})=[\bm{0}_{2\times 2};I_2]$.

We highlight that the \cobs~depends on both position and orientation, unlike $\mathcal{W}$-space obstacles, which depend only on position. Thus, high-order CBF \cite{Xiao.Belta.TAC22} or auxiliary-variable CBF \cite{Liu.etal.23} is unnecessary to handle the high relative degree constraint. We use two CLFs $V_1(\bm{x})=(\theta-\tan^{-1}(\frac{y_d-y}{x_d-x}))^{2}$, $V_2(\bm{x})=(v-v_d)^{2}$ (similar to \cite{Xiao.Belta.TAC22}) to reach the goal.

As discussed in Sec.~\ref{sec: 5d-cso}, unlike the purely translational case, we do not have an analytical expression of \cobs~in terms of system state $\bm{x}$. Instead, at each time step, we compute \cobs~from the robot's current configuration (linear time; see Lem.~\ref{lem: properties}), and approximate the derivative of its bounding constraints numerically. We consider two scenarios: (1) the robot starts in an unsafe set (Fig.~\ref{fig: unsafe_obstacle_combined}-Left), and (2) the robot navigates among multiple obstacles (Fig.~\ref{fig: unsafe_obstacle_combined}-Right).

\subsubsection{Unicycle Starting in Unsafe Set (In Collision)}\label{sec: case-unicycle-collision}
The robot starts in collision with the obstacle; thus, the origin lies within the \cobs~in \mdspace~(see \nth{1}-column of Fig.~\ref{fig: unsafe_obstacle_combined}). Note that the CBF's use of SDF between two polygonal sets enables the robot to leave the unsafe set. The parameters used in this case are $\bm{x}_d=(10,8), v_d = 2$m/s, $u_{1,{\max}}=-u_{1,{\min}}=5$ rad/s, $u_{2,{\max}}=-u_{2,{\min}}=8$m/s$^2$, $d_\text{safe}=0.02, p=8, c=5, \gamma = 0.8, \epsilon=10^{-6}$. The robot trajectory is shown in Fig.~\ref{fig: unsafe_obstacle_combined}, which demonstrates the proposed framework is able to bring the robot starting in the unsafe set back to the safe set. The computation time of each control loop is 7.85 ms on average.

\subsubsection{Unicycle Collision Avoidance with Multiple Obstacles} \label{sec: case-multiple}
For the multiple obstacles scenario, the safety is guaranteed by ensuring that safety constraint is satisfied w.r.t. each Obstacle $\mathcal{O}_i, i\in\{1,\ldots,4\}$, as mentioned in Rem.~\ref{rem2}, i.e.,
$ 
    h_i(\bm{x})=\text{sd}(\mathcal{R}(\bm{x}),\mathcal{O}_i)-d_\text{safe}=\text{sd}(\bm{0}_c,\mathcal{O}^\mathcal{C}_i(\bm{x}))-d_\text{safe}\geq0.
$ 
The associated CBF constraints are enforced in the CLF-CBF-QP as $L_fh_i(\bm{x})+ L_gh_i(\bm{x})\bm{u}+\gamma h_i(\bm{x})\geq \epsilon$
The parameters used are the same as in the previous setting, except for $\bm{x}_d=(3,13.5), v_d = 1.8$m/s, $p=10, c=6, \gamma = 0.3, \epsilon=0$. 

Our CBF framework allows the robot to navigate safely through a simple 
passage while reaching the goal position (Fig.~\ref{fig: unsafe_obstacle_combined}-Right-Top).  The time evolution of each CBF $h_i(\bm{x})$ (Fig.~\ref{fig: unsafe_obstacle_combined}-Right-Bottom) shows that all values remain positive throughout the trajectory, indicating that the robot is collision-free w.r.t each obstacle over time.
The average computation time per iteration is 19.63 ms. 
\removeForSpace{%
This increase versus the single-obstacle setting is due to the simultaneous enforcement of constraints from all obstacles and the need to solve five QPs (four for distance computation and one for CLF-CBF-QP) per iteration.
As suggested in \cite{Thirugnanam.etal.ACC22}), a hybrid approach that uses hypersphere approximations for distant obstacles and polytopes for nearby obstacles may increase computational efficiency while still enabling obstacle avoidance in tight environments. 
}


\section{Conclusion and Future Work}
This paper presents a new CBF framework based on the exact SDF between polygons, obtained via convex optimization in \mdspace. We have shown that, in the minimum distance case, the proposed CBF is differentiable a.e. and always has a well-defined subdifferential. For penetration depth, we empirically show that the proposed optimization reformulation is compatible with our CBF framework. We validate our approach in three scenarios that highlight non-conservative maneuvering, collision recovery, and applicability to multi-obstacle environment. \textcolor{black}{
Future work includes deriving analytical expressions of \cobs~bounding constraints for rotational cases and performing a formal nonsmooth analysis for the penetration depth case.}




\section*{ACKNOWLEDGMENT}
We thank David Mount at the University of Maryland for his valuable insights on Minkowski operation, and Akshay Thirugnanam at the University of California, Berkeley for helpful discussions on robot parameterization. This work was partially funded by the NSF under grant NSF GCR 2219101.
\bibliographystyle{ieeetr}        
\bibliography{references}  

\end{document}